\documentclass[sigconf]{acmart}

\usepackage{microtype}
\usepackage{graphicx}
\usepackage{subfigure}
\usepackage{booktabs} % for professional tables

\usepackage{bbold}
\usepackage{times}
\usepackage{latexsym}
\usepackage{amsmath}
\usepackage{amsfonts}
\usepackage{graphicx}
\usepackage{multirow}
\usepackage{verbatim}
\usepackage{tikz}
\usepackage{hyperref}
\usetikzlibrary{fit,positioning}
\usepackage{comment}
\usepackage{cleveref}
\usepackage{appendix}
\usepackage{hyperref}

\usepackage{algorithm}
\usepackage{algorithmic}
\usepackage{multirow}
\usepackage{multicol}

\newcommand{\fedavg} {\textsf{Fed\-Avg}}
\newcommand{\fedsgd} {\textsf{Fed\-SGD}}
\newcommand{\fedft} {\textsf{Fed\-FT}}

\newcommand{\fedavgldp} {\textsf{Fed\-Avg\-lDP}}

\newcommand{\vaers} {\textsf{VAERS}}

\usepackage{relsize}
\newcommand*{\defeq}{\stackrel{\mathsmaller{\mathsf{def}}}{=}}

\newcommand{\ind} {\emph{Ind}}
\newcommand{\fl} {\emph{FL}}
\newcommand{\ftfl} {\emph{FT-FL}}
\newcommand{\dpfl} {\emph{DP-FL}}
\newcommand{\ftdpfl} {\emph{FT-DP-FL}}

\begin{document}

%\title{Federated Learning With Mixture-of-Experts}
%\title{Personalized Cross-Silo Federated NER for Adverse Event Detection using Mixture of Experts}
\title{Private Cross-Silo Federated Learning for Extracting Vaccine Adverse Event Mentions}


% \begin{verbatim}
% \author{...}
% \orcid{...}
% \affiliation{...}
% \affiliation{...}
% \email{...}
% \author{...}
% \orcid{...}
% \affiliation{...}
% \email{...}
%\begin{verbatim}
\author{Pallika Kanani \ \ \ \ \ \ \ \ \  Virendra J. Marathe\ \  \ \ \ \ \ \ \ Daniel Peterson}
\affiliation{Oracle Labs}
\email{{pallika.kanani,virendra.marathe,daniel.peterson}@oracle.com}
\author{Rave Harpaz  \ \ \ \ \ \ \ \ \ Steve Bright}
\affiliation{Oracle}
\email{{rave.harpaz,steve.bright}@oracle.com}
%\end{verbatim}

\pagestyle{plain}


%\keywords{datasets, neural networks, gaze detection, text tagging}


\begin{abstract}

  Federated Learning (FL) is quickly becoming a goto distributed
  training paradigm for users to jointly train a global model without physically 
  sharing their data.  Users can indirectly contribute to, and
  directly benefit from a much larger aggregate data corpus used to
  train the global model.  However, literature on successful
  application of FL in real-world problem settings is somewhat
  sparse.  In this paper, we describe our experience applying a FL
  based solution to the Named Entity Recognition (NER) task for an
  adverse event detection application in the context of mass scale
  vaccination programs.  We present a comprehensive empirical analysis
  of various dimensions of benefits gained with FL based training.
  Furthermore, we investigate effects of tighter \emph{Differential
    Privacy (DP)} constraints in highly sensitive settings where
  federation users must enforce \emph{Local DP} to ensure strict
  privacy guarantees.  We show that local DP can severely cripple the
  global model's prediction accuracy, thus disincentivizing users from
  participating in the federation.  In response, we demonstrate how
  recent innovation on \emph{personalization} methods can help
  significantly recover the lost accuracy.  We focus our analysis on
  the Federated Fine-Tuning algorithm, \fedft, and prove that it is
  \emph{not} PAC Identifiable, thus making it even more attractive for
  FL-based training.

\end{abstract}

\maketitle

\section{Introduction}

Federated Learning (FL) is a distributed ML paradigm that enables
multiple users to jointly train a shared model without sharing their
data with any other users~\cite{bonawitz19,konecny15}, offering
advantages in both scale and privacy.  In FL, multiple users wish to
perform essentially the same task using ML, with a model architecture
that is agreed upon in advance.  Each user wants the best possible
model for their individual use, but often has a limited budget for
labeling their own data. Pooling the data of multiple users could
improve model accuracy, because accuracy generally increases with
increased training data. However user data cannot be shipped to a
common model training facility due to bandwidth limitations or data
privacy concerns.  As a result, users locally train the shared
(global) model on their local data, and thereafter send the updated
model to the \emph{federation server}.  The federation server
aggregates updates received from its users to improve the global model
for all users.

Although the initial focus of FL has been on targeting millions of
mobile devices~\cite{bonawitz19}, also called \emph{cross-device FL},
the benefits of its architecture are evident even for institutional
settings, also called \emph{cross-silo FL}~\cite{kairouz19}.  While
cross-device FL is concerned with both bandwidth consumption and data
privacy, cross-silo federations and their users are considered well
equipped with resources to handle bandwidth concerns, and data privacy
is the primary objective.  Our work focuses on the cross-silo FL
setting.

Today our world grapples with safely rolling out massive scale
vaccination programs to end a pandemic.  Understanding adverse events
related to these vaccines is critically important. These adverse
events are often expressed in free text form, such as social media
posts and reports provided to health care agencies and pharmaceutical
companies. Currently, mentions of specific adverse events are extracted and coded manually, which is a time consuming, expensive and non-scalable process. Therefore, Machine Learning (ML) based methods to extract
named entities (adverse events) automatically from such unstructured
data are highly desirable.

Typically, more training data yield more
accurate models.  Unfortunately, collecting human annotations for
building such Named Entity Recognition (NER) models is expensive, and
particularly challenging given the need to maintain privacy of health
records. One way to overcome this data scarcity issue would be for
various agencies to share their data to build a joint model with
combined data. However, privacy concerns, government regulation and
data use agreements might not allow the data to leave individual
organizational or geographical silos.  Sharing user data with other
users is absolutely not an option in these settings.

\begin{comment}
Consider a group of hospitals wanting to collaborate on building an ML
model. Each hospital might have only a small amount of labeled
training data, so pooling the data can help build a more accurate
model. However, the privacy of each of the patient’s data is paramount
and they might be contractually obligated to not release it beyond
their institutional boundary. Furthermore, each of the hospitals might
draw patient populations from different demographics, leading to
variations in the data distribution. This is a fairly common scenario
which also arises in many other contexts. Consider an ML service being
offered to enterprise customers: the number of users of an ML service
may be much smaller, but the data may contain competitive
intelligence.  Yet another situation in which Federated Learning can
help is overcoming geographical barriers, under which the data in one
geography can not leave a geopolitical boundary

\begin{figure}
%\ffigbox{\TopFloatBoxes
\includegraphics[width=3.0in]{figures/bar_plot_with_error_bars.png}
%}{
\caption{ Classifier accuracy (higher is better) on a spam classification dataset
  comprising $15$ users cooperating in a FL setting (more details on
  the benchmark in Evaluation section). Introduction of DP-induced noise
  significantly compromises accuracy.}
\label{spam-b-results-0}
%}
%\end{floatrow}
\end{figure}
\end{comment}

Cross-silo FL makes perfect sense to address such problems.  Each
vaccine provider's data remains in its private \emph{silo}.  At the
same time, the provider can collaborate with other providers on a FL
framework to collectively improve the NER model used for adverse event
detection.  Everyone benefits without violating data privacy.
%We call this kind of data privacy, \emph{Tier 1} data privacy.
More specifically, for institutions participating in a federation as
users, restricting data movement helps fulfill contractual obligations
with their customers and comply with legal regulatory constraints on
data movement~\cite{ccpa,gdpr}.

However, restricting the provider's training data to its private silo
does not guarantee complete privacy.  Recent works have demonstrated
that the data can indirectly leak out through model updates shipped by
users to the federation server~\cite{bagdasaryan20,melis18,nasr19}.
To combat this problem, researchers have proposed the addition of
Differential Privacy (DP) ~\cite{dwork06a,dwork14,dwork06} to
FL~\cite{abadi16,geyer17,konecny16,mcmahan17}.

Informally, DP aims to provide a bound on the variation in the model's
output based on the inclusion or exclusion of a single data point used
in its training set.  This is done by introducing precisely calibrated
noise in the training process.  The method of noise calibration and
injection varys between implementations~\cite{abadi16,mcmahan17}, but
is always structured to enforce the precise formal DP guarantee, which
we define in~\autoref{sec:background}. We will refer to this process
as ``DP inducing noise injection'' henceforth.  This noise makes it
difficult, even impossible, to determine whether any particular data
point was used to train the model.

In settings where the federation server is trusted, DP enforcement is
delegated to the federation server~\cite{mcmahan17}.  However, in
settings where users do not trust even the federation server, a
stricter form of DP, called \emph{Local DP}, is
enforced~\cite{kasiviswanathan08}.  While all this noise is structured
to enforce formally provable privacy guarantees for each training data
point~\cite{dwork06}, it can significantly degrade accuracy of model
predictions. This degradation may happen to an extent that
disincentivizes users from participating in the federation -- the
global (noisy) model performs worse than a user-resident local model
trained just on the user's dataset, which we call the
\emph{individual} model.

Another instance where the global model may perform worse than the
individual model for a user is when the user's data distribution is
different from most of the users, or the users collectively have
non-IID training data~\cite{hsieh19,li20}.  There is a rapidly growing
body of FL \emph{Personalization} literature to address this
problem~\cite{dinh20,fallah20,liang20,mansour20,peterson19,yu20}, a
handful of which addresses model degradation due to DP induced
noise~\cite{peterson19,yu20}.

We are interested in applying this body of work to real-world problem
settings.  The health care sector is one such application domain that
can leverage FL in significant ways.  Indeed there is rapidly growing
awareness and investment at world-wide scale including
consortiums~\cite{melloddy} and public-private
partnerships~\cite{imi}.  This is accompanied by the beginnings of
applied research in this sector~\cite{li20a}.

In this paper, we case study application of FL to the problem of
vaccine adverse event detection, the first of its kind to the best of
our knowledge.  Importance of such a study cannot be understated in
today's pandemic stricken world. The unprecedented speed at which new
vaccines have been rolled out, it is crucial to automatically extract
mentions of adverse events related to these vaccines from patient
reports. We study implications of applying FL to the Vaccine Adverse
Event Reporting System (\vaers) dataset that we have annotated and
partitioned by vaccine manufacturers. Each vaccine manufacturer acts
as a federation user whose dataset is siloed in its private sandbox;
all these sandboxes participate in our FL framework over multiple
training rounds.

Our experiments reveal several interesting insights including general
effectiveness of FL on model performance, effects of local DP
enforcement on model performance, and the value of personalization
techniques to incentivize users to participate in FL.  In particular,
we show that FL improves average F1 value by 37.43\% over the
individual model, while enforcement of local DP (DP-FL) degrades the
FL model's average F1 by 25.17\%. For one of the users, this
degradation is so severe that the private FL model F1 is worse by
45.55\% when compared with the individual model F1.  This clearly
makes DP-FL a non-starter for some users to join the federation.  We
study FL with \emph{Fine-Tuning} (FT-FL)~\cite{yu20}, a
personalization approach that fine-tunes the global model at each user
\emph{after} the entire FL training process completes.  Interestingly,
contrary to prior work~\cite{yu20}, simply augmenting fine-tuning to
FL does not result in prediction accuracy improvement for the
federation users.  Instead, user accuracy degrades in most cases.
However, somewhat surprisingly, fine-tuning in the presence of DP
(FT-DP-FL) boosts user accuracy by 24.88\%, compared to the individual
model, to strongly incentivize users to join and stay with the
federation.  We also observe that vaccine reports related to different
manufacturers have slightly different vocabulary (e.g. mentions of
different vaccine names), and different distributions of adverse
events, which aid FT-DP-FL in effectively recovering lost accuracy.

Even more interestingly, our findings indicate a unique
\emph{incentive structure} for users to join the federation.  In
particular, we find that users with small amount of training data,
a.k.a.~\emph{small} users, have a strong incentive to join and stay
with the federation even when local DP is enforced without
fine-tuning.  This is because the user's private dataset is so small
that any locally trained individual model performs poorly.  In
contrast, even the global model that is degraded because of DP
inducing noise performs significantly better than the user's
individual model.  In short, small users have virtually no incentive
to leave the federation, and may not require additional layers of
personalization to improve the global model as long as there are
enough participants in the federation.

For users with larger amount of data, the narrative is quite
different.  In particular, we observe that the global model's
degradation due to DP inducing noise is significant enough to
disincentivize those users from participating in the federation.  As a
result, if they opt for the additional layer of privacy through DP,
the importance of personalization based enhancements, which salvage
the accuracy lost due to DP inducing noise, cannot be understated.

While personalization is gaining traction in the FL research
community, we address a recent concern, called \emph{PAC
  Identifiability}~\cite{london20}, on preserving privacy of the
user-private model (typical of personalization approaches).  In
particular, we present a methodology to prove PAC
\emph{non}-identifiability for a given personalization approach, and
prove that the approach of our choice, FT-FL (and FT-DP-FL), is
\emph{not} PAC identifiable.

In summary, this paper makes the following contributions:

\begin{itemize}
  \item We present the first comprehensive study, to the best of our
    knowledge, on application of FL to the vaccine adverse event
    detection task in the field of pharmacoviligence on real-world
    data -- the \vaers\ dataset.
  \item Our study examines benefits of FL based training, along with
    its robustness to user participation.
  \item We examine challenges posed by enforcement of stricter
    differential privacy, to the extent that may disincentivize users
    from participating in a federation.
  \item We show that, unlike prior work~\cite{yu20}, simply augmenting
    FL with personalization techniques, such as the aforementioned
    fine-tuning (FT-FL), does not necessarily improve prediction
    accuracy for FL users.  In fact, it degrades prediction accuracy
    in our experiments.  However, somewhat surprisingly, the same
    techniques (FT-DP-FL) turn out to be highly effective in
    recovering lost accuracy due to DP inducing noise injection.  We
    furthermore show that personalization is robust to user
    participation uncertainties (e.g. users dropping out).
  \item We report an interesting new \emph{incentive structure}
    amongst users participating in the federation, where users with
    small amount of training data are strongly incentivized to join
    and stay with the federation, whereas users with somewhat larger
    amounts of data require enhancements, such as FT-DP-FL, to
    overcome the pitfalls of DP inducing noise injection.
  \item Another surprising finding in our study is that fine-tuning
    based personalization is highly resilient to increasingly tighter
    margins for the differential privacy budget ($\epsilon < 1$).
  \item Finally we provide a novel methodology to prove that a
    personalization approach is \emph{PAC non-identifiable}.
    Moreover, we prove that the personalization approach of our
    choice, FT-FL (and FT-DP-FL), is PAC non-identifiable.
\end{itemize}

The rest of the paper is structured as follows: We discuss background
material and related work in~\autoref{sec:background}. The VAERS
system used as the basis of this study is described in
~\autoref{sec:vaers}. We describe our NER model used in an adverse
event detection system, along with our FL framework and the
personalization approach we use in~\autoref{sec:method}.  Our
comprehensive experiments and their analysis appears
in~\autoref{sec:experiments}.  This is followed by our formal proof on
PAC non-identifiability for our personalization approach, FT-FL (and
FT-DP-FL) in~\autoref{sec:pac-non-identifiability}, and conclusion
in~\autoref{sec:conclusion}.

 \section{Background}
\label{sec:background}

\paragraph{Federated Learning (FL)}

In FL, a federation server initializes a global model and ships it to
all participating users thereby initiating distributed training.
Training happens over multiple rounds.  In each round, each user, on
receiving the the global model re-trains the model on its private data
and sends back the resulting parameter updates to the federation
server.  The federation server aggregates updates from all users
applying them to the global model, and then ships the revised model
back to the users.  The most widely used method of aggregation is
\fedavg~\cite{konecny15,mcmahan16}, where user parameters updates are
averaged at the federation server and applied to the global model.
Formally, \fedavg\ solves the following optimization problem:
\begin{equation}
\begin{aligned}
  \min_{w \in \mathcal{R}^d} f(w) \hspace{0.25in} where,
  f(w) \defeq \frac{1}{n} \sum_{i=1}^n f_i(w)
\end{aligned}
\end{equation}
The function $f_i = \mathcal{l}(w;x_i,y_i)$ represents the local loss
for each of the $n$ federation users on the model $w$ using the user's
private data ${x_i,y_i}$.

Figure~\ref{fl-picture} shows the overall FL architecture.  Users can
dynamically join the federation or drop out.  The framework is
structured to be resilient to such changes.  Noting privacy concerns,
more recent work has proposed addition of differential privacy to
FL~\cite{geyer17,konecny16,mcmahan16}.

\begin{figure}
%\ffigbox{\TopFloatBoxes
\includegraphics[width=\columnwidth]{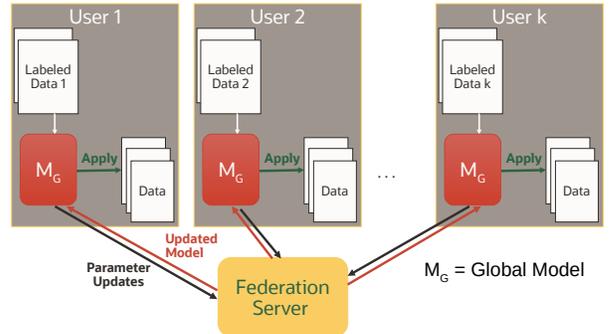}
%}{
\vspace{-1.0in}
\caption{The Federated Learning setting. $M_G$ is the global model the
  federation server sends to users, each of which re-trains $M_G$ on
  its private data and sends the updated model parameters back to the
  federation server.}
\label{fl-picture}
%}
%\end{floatrow}
\end{figure}

\paragraph{Differential Privacy (DP)}

\emph{Differential privacy}~\cite{dwork06} is a mathematically
quantifiable privacy guarantee for a data set used by a computation
that analyzes it.  While it originally emerged in the database and
data mining communities, triggered by privacy concerns in Machine
Learning
(ML)~\cite{fredrikson15,fredrikson14,hitaj17,korolova10,shokri17,tramer16},
DP has garnered enormous traction in the ML community over the last
decade~\cite{abadi16,carlini19,chaudhuri11,appledp17,dimitrakakis17,fredrikson14,fredrikson15,park16,park16a,sarwate13}.

In DP, the privacy guarantee applies to each individual item in the
data set and is formally specified in terms of a pair of data sets
that differ in at most one item.  Specifically, consider an algorithm
$A$ such that $A: D \mapsto R$, where $D$ and $R$ are respectively the
domain and range of $A$.  Now consider two data sets $d$ and $d'$ that
differ from each other in exactly one data item.  Such data sets are
considered \emph{adjacent} to each other in the DP literature.
Algorithm $A$ is said to be $(\varepsilon,\delta)$-differentially
private if the following condition holds true for all adjacent $d$ and
$d'$ and any subset of outputs $O \subseteq R$:

\begin{equation}
  P[A(d)~\epsilon~O]~\leq~e^{\varepsilon}~P[A(d')~\epsilon~O]~+~\delta
\end{equation}

\begin{comment}
$\varepsilon$ represents the upper (multiplicative) bound for the
difference between the probabilities with which $A$ generates an
output from $O$ from inputs $d$ or $d'$.  $\delta$ represents the
probability with which the $\varepsilon$ bound may not hold
(information on the input data set may be leaked).
$\varepsilon$-differential privacy is guaranteed when $\delta\ = 0$.
\end{comment}

Enforcement of DP typically translates into introduction of a
``correction'' in algorithm $A$ to ensure that the differential
privacy bound holds for any two adjacent inputs.  This correction is
commonly referred to as the \emph{noise} introduced in the algorithm,
its input, or output to ensure that the
$(\varepsilon,\delta)$-differential privacy bound holds.  While a
disciplined introduction of noise guarantees DP, the noise itself
leads to accuracy degradation in the output produced by $A$.  In the
context of ML, the algorithm is a model being trained using sensitive
private data sets, and accuracy degradation can significantly hamper
the model's utility.

In the FL context, one can enforce DP using two distinct approaches:
(i) \emph{Global DP}, also called \emph{Central DP} in the
literature~\cite{mcmahan17,zhu20}, where users fully trust the
federation server to enforce DP.  The server in turn enforces DP to
obfuscate the \emph{participation} of each user.  (ii) \emph{Local
  DP}~\cite{dp-apple17,duchi13,kasiviswanathan08,truex20}, where users
do not trust the federation server, and enforce DP on the updates
shipped back to the server.  This method of DP enforcement typically
guarantees privacy to finer granularity of individual training data
points~\cite{liu20}.  A FL deployment that is considering addition of
DP must carefully evaluate implications of above approaches for the
intended use cases.

\paragraph{Personalization in FL}

The basic FL algorithm, \fedavg, assumes IID training data
across all FL users.  In fact, it is known to be quite effective in
practice for such data distributions.  However, \fedavg\ may perform
poorly in the presence of non-IID user data~\cite{hsieh19,li20}.  A
recent flurry of research addresses this problem using
\emph{personalization}
techniques~\cite{dinh20,fallah20,liang20,mansour20,peterson19,yu20}
that specialize training at each user, typically in the form of
training an additional local model, or letting the local copy of the
global model ``drift'' from the global model in a constrained fashion.
This enables the local model to fit better to the user's local data
distribution thereby delivering a better performing model.  While the
work in personalization is promising, recent research shows that such
local models are vulnerable to being \emph{PAC Identifiable} (leaked
out) to an adversarial federation server~\cite{london20}.

\paragraph{Adverse Event Mention Extraction}

By some estimates, adverse drug reactions are among the leading causes
of death in the developed world. Reports of adverse events are a
critical source of information for tracking and studying adverse
events associated with medicinal products. However, portions of the
sought information is only available in unstructured format.  The use
of and necessity of automated methods for extracting mentions of drug
adverse events from unstructured text is widely recognized in
pharmacovigilance ~\cite{harpaz14}. Several different
genres of text are tackled in this line of research, including social
media ~\cite{gurulingappa12, Korkontzelos16}, biomedical literature
~\cite{leaman10, winnenburg15}, clinical narratives ~\cite{haerian12,
  lependu13} and drug labels ~\cite{roberts17}. More recently, use of state of the art
deep learning technology for Named Entity Recognition (NER) have been proposed ~\cite{giorgi18}.

\section{Vaccine Adverse Event Reporting System}
\label{sec:vaers}

Drug and vaccine safety surveillance relies predominantly on
spontaneous reporting systems. These systems are comprised of reports
of suspected drug/vaccine adverse events (potential side effects)
collected from healthcare professionals, consumers, and pharmaceutical
companies, and maintained largely by regulatory and health
agencies. Among other, these systems are used to detect possible
safety problems – called ``signals'' – that may be related to a
vaccination or the consumption of a drug. In the US, the prominent
surveillance system for vaccines is the U.S. Centers for Disease
Control and Prevention (CDC) and the Food and Drug Administration
(FDA) Vaccine Adverse Event Reporting System (\vaers), created in
1990.
 
The \vaers\ data (de-identified) is publicly available in structured
format. Each \vaers\ report includes the name of (and additional
information about) the administrated vaccine, a list of adverse events
related to the vaccine, dates, and limited demographic information
about the patient receiving the vaccine (e.g., age,
gender). Importantly, the report also includes a textual narrative
describing the adverse event. For example,
 
\emph{Shortly after patient was vaccinated, she started to feel an
  itching, tingling feeling in her throat.  Fearing that it was an
  allergic reaction, I called 911.  The patient remained alert,
  talking and breathing normally until paramedics arrived, though she
  stated that she started to feel additional tingling in her arms and
  chest.}

Most of the data collected in VAERS is currently processed by humans for downstream applications.  With the rapidly increasing volume of such data this human effort is becoming prohibitive and calls for the increased use of automated methods such as NER. In addition, pharmacovigilance data such as that available in and similar to VAERS originates from private siloed sources,  which motivate the need for privacy preserving distributed approaches such as FL.

\section{Model and Framework}
\label{sec:method}

\subsection{NER based on Recurrent Neural Networks}
 
%TODO clean this up. 
The recurrent neural network (RNN) architecture we used to perform NER
is based on a commonly applied BiLSTM architecture.
%displayed in Figure 1.
The architecture consists of three major components: (1) a word
representation layer made of word embeddings, (2) two stacked layers
of bidirectional long short-term memory (LSTM) cells, and (3) a
feedforward layer that performs the final BIO sequence labeling.
 
Pre-trained word embeddings were used to seed the network’s word
embedding layer. These were generated using Word2Vec applied to the
sentences comprising the VAERS NER dataset described in ~\autoref{sec:experiments}. Dropout regularization was implemented between each of the three major network
components. The dropout rate was 0.4.
 
The network was implemented on PyTorch6 and trained using stochastic
mini-batch gradient descent with the Adam optimizer for a pre-defined
number of iterations. Each iteration processed a batch of 256 randomly
selected sentences. The network was trained for a total of 20 epochs,
each epoch consisting of number of sentences in the training set /
batch size iterations.

\subsection{Federated Learning Framework}

We have implemented our own FL simulation framework, on PyTorch6,
that hosts the federation server and users on the same computer.  The
framework supports several federated aggregation protocols, including
\fedavg\ and \fedsgd~\cite{konecny15}, of which we use \fedavg\ in our
evaluation.  The framework is extendable to support other custom
aggregation
protocols~\cite{dinh20,fallah20,liang20,peterson19,yu20}. 

%Our framework also supports both local and global DP.

\subsubsection{Trust Model Considerations and Differential Privacy}

The decision to train a ML model using the FL framework requires
careful analysis of privacy considerations for users' data.  More
specifically, the \emph{meaning} of the term ``data privacy'' in a
given setting needs to be precisely understood since it has profound
implications on techniques required to enforce the desired data
privacy.  For instance, in some settings, simply restricting user data
to its private silo is sufficient for the use case.  On the other
hand, in settings involving highly sensitive private data (e.g. health
records of individuals), it may be desirable to ensure that even the
parameter updates shipped from the user silo to the federation server
cannot be reverse engineered by any means, external to the user, to
determine the user's training data records.  Ultimately, the level of
privacy protection must be agreed upon by all parties involved.  While an exhaustive treatment of a
taxonomy of such \emph{trust models} in FL is beyond the scope of this
paper, we assume that personal health records describing an adverse reaction to a vaccine are highly sensitive
private material.  Consequently, they must be protected using
techniques guaranteeing the strictest data privacy.

In the FL setting, these data records would be hosted in a participating
pharmaceutical company's silo.  The pharmaceutical company's silo
performs the role of a user in the federation.  We view Differential
Privacy (DP) as an appropriate tool to enforce privacy guarantees to
individuals' health records.  However, more careful analysis of how DP
is enforced in FL settings is required.

We have already discussed two choices available: (i) global DP, and
(ii) local DP.  Other technologies such as secure multi-party
computation~\cite{yao86} and homomorphic encryption~\cite{gentry09}
may be worth considering, but are beyond the scope of this work.
Additional security technologies such as end-to-end encryption may be
necessary to augument to the DP solution, but is also outside the
scope of this work.  The question with DP choices is: which approach
is applicable in our setting?  In principle, global DP can be used in
such a setting, provided the FL service provider is under a legally
binding contract to preserve privacy of the parameter updates, thus
indirectly preserving privacy of the training data, received from the
user.  However, such arrangements may not be applicable to contexts
where movement of such noise-free parameter updates is restricted by
government regulations or due to disparate privacy laws across
different jurisdictions~\cite{ccpa,gdpr}.  As a result, local DP may
be far more attractive in such tightly constrained settings.

As stated earlier, adverse reactions to vaccines is highly sensitive
private data.  Futhermore, the incentive to collaborate across
multiple jurisdictions, even continents, is extremely alluring due to
the ongoing pandemic at a global scale.  As a result, local DP is the
privacy enforcement mechanism that is best suited to our setting.

To enforce local DP, we use the algorithm proposed by Abadi et
al.~\cite{abadi16} that injects gaussian noise (calculated using their
moments accountant algorithm) in parameter gradients during local
training at each user.  Nosiy gradients lead to noisy parameter
updates, which are eventually shipped from the user to the federation
server.

Interestingly, since users can possess datasets with different sizes,
the computed noise, which is a function of the dataset size, varys
considerably from user to user.  For instance, the noise introduced
for a user with a handful of data points is much higher than the noise
introduced by a user with a much larger private dataset.  However,
\fedavg\ smoothes out the noisy updates through the parameter
aggregation process (averaging, in our case).  The resulting model
that each user receives is much more robust.

%% The VAERS dataset represents a use case where data privacy is of
%% paramount importance.  We believe that sensitivity of health records
%% and corresponding government regulations on the pharma industry imply
%% stricter privacy enforcement mechanism, such as local DP, in the FL
%% setting.  As a result, we would ideally want to enforce local DP
%% (\fedavgldp).  However, we are generally interested in understanding
%% effectiveness of FL on training our NER model on the VAERS dataset.
%% As a result, we train the model on vanilla \fedavg\ as well as
%% \fedavg\ with global DP (\fedavggdp).

\subsubsection{Personalization through Fine Tuning}

The main allure of FL for a user is the promise of significant
prediction accuracy improvements over a locally trained
\emph{individual} model.  While parameter aggregation through FL can
significantly improve accuracy of the global model, introduction of
noise to enforce DP can severely compromise that improvement.  The
degradation can be severe enough to make users reconsider their
decision to join the federation, and deter new users from joining the
federation.  Furthermore, data distributions across users may have
significant side effects on the global model's prediction accuracy: If
a user's dataset has a significantly different distribution than most
of the federation users, the global model may perform worse than a
locally trained individual model.  If users of a federation have
non-IID data, the resulting global model may be
ineffective~\cite{li20}.

Many researchers have recently proposed different forms of
\emph{personalization} approaches to remedy the disparate data
distribution
problem~\cite{arivazhagan19,deng20,dinh20,fallah20,haddahpour19,
  hanzely20,jiang19,liang20,mansour20,peterson19,smith2017,yu20}.
Just two of these works~\cite{peterson19,yu20}, to the best of our
knowledge, propose personalization approaches as solutions to model
degradation due to DP inducing noise.  Among the proposed
personalization approaches, we focus on FL with \emph{Fine
  Tuning}~\cite{yu20}: FT-FL for fine tuning on top of plain FL, and
FT-DP-FL for fine tuning on top of FL with local DP enforcement.  In
this approach each user continues training, without noise, the local
copy of the global differentially private model \emph{after} the FL
training process has completed.

The fine tuning based parameter updates are private to each user and
are not shared with the federation.  As a result, the fine tuned local
models may diverge from the global model at varying degrees in order
to better fit the users' private data.  While endlessly fine tuning
the global model can lead to the model converging to a locally trained
individual model, care must be taken to ensure that the fine-tuned
model does not deteriorate.  This can be achieved through standard
hyperparameter tuning techniques.

\section{Experiments}
\label{sec:experiments}

\subsection{Dataset}

We used a total of 17,841 narratives submitted to \vaers\ through the
years 2015-2017 to form the NER data set used for this study. The
narratives were automatically annotated for adverse event named
entities using the list of adverse events supplied with each
report. In total the NER data set used for this study comprised of
87,730 sentences and 39,139 annotated adverse event named
entities. Table \ref{dataset} describes the dataset in detail. In our
experiments, we split the data randomly into train, validation, tune
and test sets in the proportion 60\%, 10\%, 10\%, and 20\%
respectively. We used the validation set to decide early stopping in
the fine tuning algorithm and tuned the rest of parameters on the tune
set. We refer to ``large manufacturers'' as those with more than 1000
\vaers\ reports in this data and ``small manufacturers'' as those with
fewer reports to reflect the availability of training data in each
user's silo. In the rest of this paper, we use the terms
`manufacturer' and `user' interchangeably.

\begin{table}
\footnotesize{
\centering
\begin{tabular}{|l|c|c|c|}
\hline
\textbf{Vaccine} & \textbf{Num \vaers} & \textbf{Sentences} & \textbf{Entities}\\
\textbf{Manufacturer} & \textbf{Reports} && \\
\hline
Merck  Co. Inc. &     7638 & 42207 &15501\\
Sanofi Pasteur &     3352&12688&8071 \\
Pfizer-Wyeth &     2428 &10848&5607 \\
Glaxo-Smithkline &    2289 &11366&5186\\
Biologicals &     & & \\
Novartis Vaccines  & 1183&6664& 2648\\
And Diagnostics  & & & \\
CSL Limited & 465 &1751& 1088\\
Medimmune Vaccines & 265 &1413&498\\
Inc. & & & \\
Seqirus Inc. & 111 &426&252\\
Emergent Biosolutions &58 &206&146 \\
Berna Biotech Ltd.& 52&161&142 \\
\hline
\end{tabular}
}
\caption{ \footnotesize{
\vaers\ Dataset.`Vaccine Manufacturer’ is a field in
the public \vaers\ database that identifies the manufacturer of the
vaccine reported in the \vaers\ form. There is no relationship between
this field and the reporter. ‘Num \vaers\ Reports’ does not represent
the rate of adverse events associated with the manufacturer or its
products and cannot be used to estimate such rates. The statistics are
based on a sample of reports submitted to \vaers\ between 2015-2017
whose MedDra coded adverse events appeared in the narrative. Because
the statistics are based on a carefully selected sample, the
distribution of reports shown may not represent the true distribution
of reports associated with different vaccine manufacturers. 
}}
\vspace{-0.4in}
\label{dataset}
\end{table}

\vspace{-0.1in}
\subsection {Experimental Setup}

As the first baseline for our experiments, we train Individual models
(\ind), i.e. assume that each manufacturer only uses their own
training set, and test on their respective test set. This baseline
represents the case in which the manufacturer chooses not to
participate in the federation at all. \emph{FL} is the federated
learning model trained in a collaborative fashion across users using
the \fedavg\ algorithm. This model is then fine tuned for each user
using the protocol described in~\autoref{sec:method}, which yield a
set of models, one per manufacturer, that we call \emph {FT}. Next, we
introduce local differential privacy to the \emph{FL} model, as
described in~\autoref{sec:method}. We use $\epsilon = 2.0$ for this
first set of experiments as it is considered a fairly conservative
privacy setting in the literature~\cite{abadi16} and calculate the
sigma values suitable per user.
%, as shown in~\autoref{sigmas}.
We call this private federated learning variant \emph{DP-FL}. Finally,
we fine tune this private FL model and call it \emph{FT-DP-FL}.

The training parameters for all of these algorithms were tuned using a
separate tuning dataset. We use a learning rate of 0.01 and train all
the federated models for 20 rounds of \fedavg, with additional 20
epochs for the fine tuning variants at each manufacturer. For
evaluation, we compute the precision, recall, and F1 of each token
label on a 1-vs-all basis. The values reported are the mean F1 score
(henceforth called F1) for the labels at the beginning or inside of an
adverse event mention.

We ask the following questions as part of this study. Does
\fl\ perform better than \ind\ models across users? What happens
when differential privacy is introduced?  Does personalization help
improve accuracy over \fl\ and mitigate \dpfl{}'s accuracy loss enough
to re-incentivize users to participate in the federation?  If
fine-tuning based personalization helps mitigate accuracy loss due to
DP, how robust is it to varying parameters of DP?  Finally, we ask if
the federation is stable enough for the uncertainties of real world,
such as users dropping out?  We also analyze the incentive structure
that emerges for users with varying amounts of training data.

\vspace{-0.1in}
\subsection{Private Federated Learning with Personalization}

Figure \ref{large-fed-results} shows the F1 values for each of the
described models on the individual users' test sets. Note that the
manufacturers on the $x$-axis are sorted based on the size of their
training sets. As we
can see, the FL model consistently outperforms \ind\ models for each
of the users, including large manufacturers with a lot of training
data. As table \ref{large-fed-results-table} shows, the
amount of error reduction over the \ind\ model for each user is substantial. 
Contrary to findings by Yu et. al.~\cite{yu20}, in our case,
personalization based on fine tuning \ftfl\ performs worse than
\fl\ in most cases. As we add noise related to differential privacy to
the federated learning model, F1 values drop significantly across the
board. This makes participation for larger manufacturers in the
federation unattractive, since the \dpfl\ model ends up performing
worse than their \ind\ models. However, applying fine tuning in this
case helps bring it back up to the point, where it is again
advantageous for each party to participate in the federation. This
shows that personalization based approach can help mitigate the loss
of accuracy from introducing differential privacy.

It is interesting to note that for small manufacturers, with an
exception of one with very small amount of evaluation data, it is
always beneficial to participate in the federation, even for \dpfl,
with or without personalization.  For large manufacturers however, the
DP is only attractive in the presence of the mitigation offered by
fine-tuning based personalization (\ftdpfl).

\begin{figure}[t]
%\ffigbox{\TopFloatBoxes
\includegraphics[width=\linewidth]{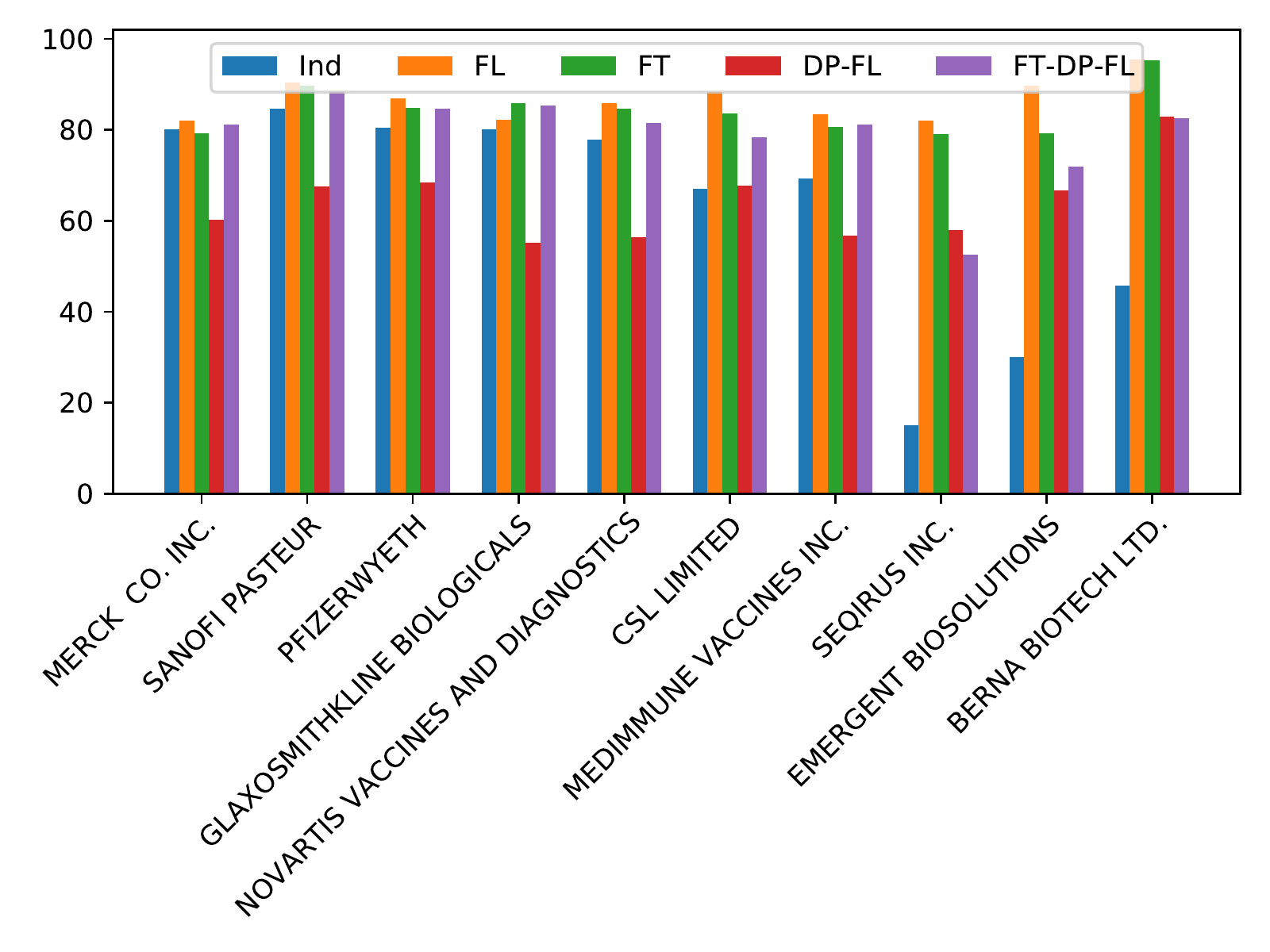}
%}{
\vspace{-0.4in}
\caption{\footnotesize{F1 per manufacturer for different methods for $\epsilon=2.0$}}
\label{large-fed-results}
%}
%\end{floatrow}
\end{figure}

\begin{table}
\footnotesize{
\centering
\begin{tabular}{|l|c|c|c|c|c|}
\hline
\textbf{Vaccine} & \textbf{Individual} &\multicolumn{2}{|c|}{\bfseries FL} & \multicolumn{2}{|c|}{\bfseries FT-DP-FL}\\
\textbf{Manufactuer} & \textbf{F1} &\multicolumn{2}{|c|}{\bfseries } & \multicolumn{2}{|c|}{\bfseries }\\\hline
& & \textbf{F1} & \textbf{Error Red.} & \textbf{F1} & \textbf{Error Red.}\\
	
%\textbf{Vaccine Manufactuer} & \textbf{Individual F1} &%& & \textbf{F1} & \textbf{Error Red.} & \textbf{F1} & \textbf{Error Red.}\\
\hline
Merck Co. Inc. & 80.10 & 82.00 & 9.55\% & 81.20 & 5.53\%\\
Sanofi Pasteur & 84.60 & 90.40 & 37.66\% & 88.40 & 24.68\%\\
Pfizer-Wyeth & 80.50 & 87.00 & 33.33\% & 84.60 & 21.03\%\\
Glaxo-Smithkline & 80.20 & 82.20 & 10.10\% & 85.30 & 25.76\%\\
Biologicals &  &  &  &  & \\
Novartis Vaccines & 77.80 & 85.80 & 36.04\% & 81.50 & 16.67\%\\
And Diagnostics &  &  &  &  & \\
CSL Limited & 67.10 & 88.50 & 65.05\% & 78.30 & 34.04\%\\
Medimmune  & 69.30 & 83.50 & 46.25\% & 81.10 & 38.44\%\\
Vaccines Inc.&  &  &  &  & \\
Seqirus Inc. & 15.00 & 82.10 & 78.94\% & 52.60 & 44.24\%\\
Emergent & 30.10 & 89.70 & 85.26\% & 71.90 & 59.80\%\\
Biosolutions&  &  &  &  & \\
Berna Biotech & 45.80 & 95.40 & 91.51\% & 82.50 & 67.71\%\\
Ltd. &  &  &  &  & \\
\hline
\end{tabular}
}
\caption{\footnotesize{
F1 and Error Reduction with Federated Learning and Private Federated Learning with Fine Tuning
}}
\vspace{-0.35in}
\label{large-fed-results-table}
\end{table}
\subsection{Robustness to Differential Privacy Noise}
Next, we study the effectiveness of personalization in recovering from
the accuracy loss resulting from differential privacy noise. We vary
the parameter $\epsilon$ and measure F1 averaged across users for two
of the algorithm variants: differentially private federated learning
(DP-FL) and the fine tuned differentially private federated learning
(FT-DP-FL). As we can see from~\autoref{effect-of-epsilon}, average F1
for DP-FL deteriorates significantly for values of $\epsilon$ less
than 2. However, even in these cases, the personalized version,
FT-DP-FL manages to retain its performance. We believe this is an
important finding that provides significant latitude to differentially
private FL frameworks to further tighten the privacy budget of
$\epsilon$ without compromising utility.

\begin{figure}[t]
%\ffigbox{\TopFloatBoxes
%\vspace{-0.1in}
\includegraphics[width=2.2in]{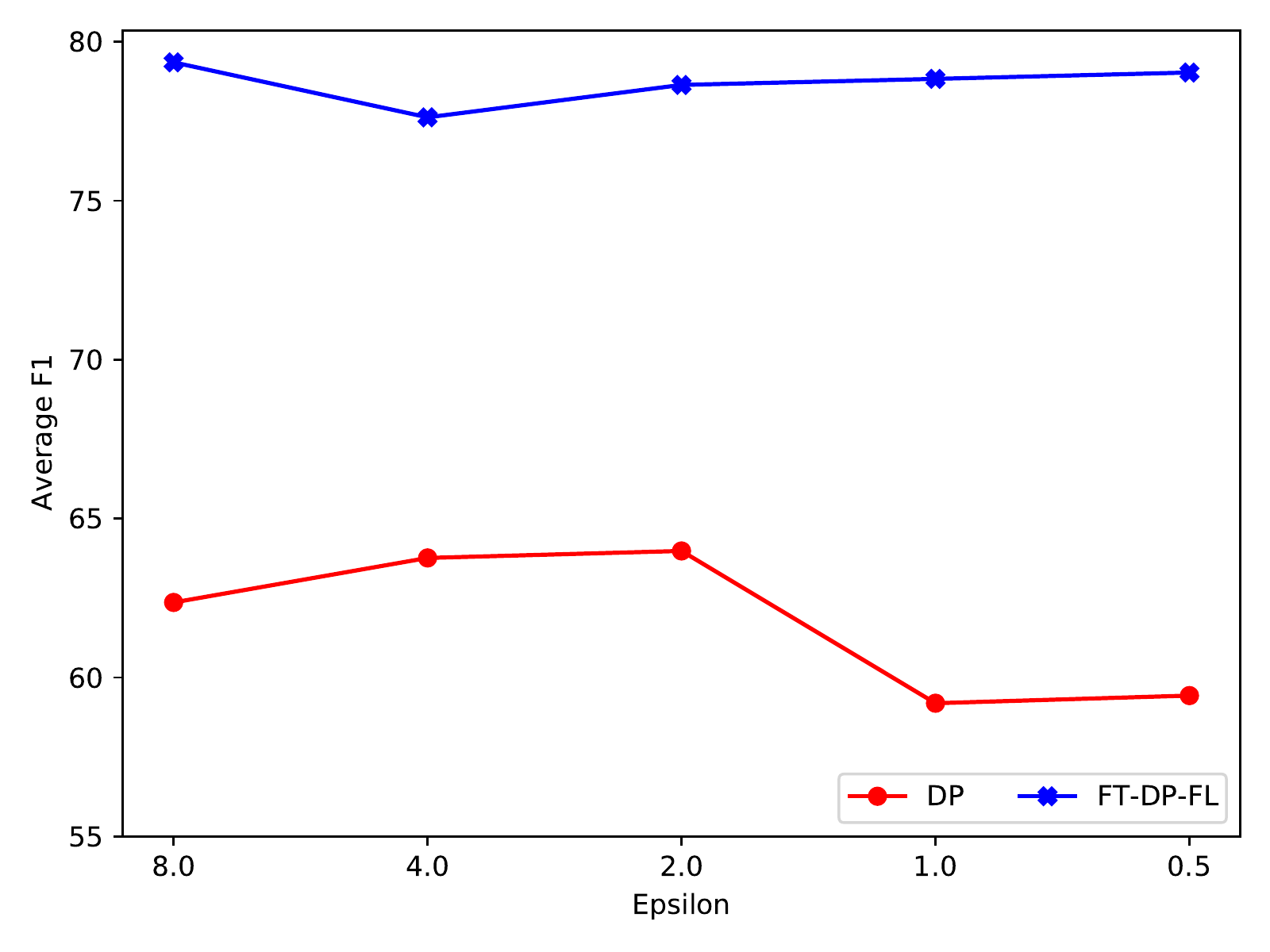}
%}{
\vspace{-0.2in}
\caption{\footnotesize{Average F1 across users for the two differentially private FL variants.}}
\vspace{-0.1in}
\label{effect-of-epsilon}
%}
%\end{floatrow}
\end{figure}

\subsection{Stability of Federation against Users Leaving}

Building a federation across organizations can be challenging in the
real world due to a variety of factors. For instance, users may
discontinue their participation in the federation.
%There is always the possibility that some of the parties may not
%continue to participate.
We simulate this scenario and study the effect of
one of the manufacturers leaving the federation. As we can see from
Tables~\ref{leave-one-out-FL} and~\ref {leave-one-out-FT-DP-FL}, both
federated learning and private federated learning with fine tuning are
fairly stable against such a change, with the exception of a few
manufacturers with very small amount of training and test data. In
other words, no single manufacturer has disproportionally large impact
on the overall accuracy gains from participating in the federation.

\begin{table}
\footnotesize{
\centering
\begin{tabular}{|c|c|c|c|c|c|c|c|c|c|c|}
\hline
&M1 & M2 & M3& M4& M5& M6& M7& M8& M9& M10\\\hline

M1&0.0&0.9&1.8&0.4&1.0&2.1&1.8&0.4&1.0&0.0\\

M2&-0.3&0.0&0.4&0.5&1.4&1.6&1.6&-0.4&3.2&-1.5\\

M3&-0.1&0.5&0.0&0.1&0.1&0.9&1.4&1.9&1.0&-1.5\\

M4&-0.6&0.8&0.2&0.0&2.6&-0.2&3.5&1.3&1.0&0.0\\

M5&-0.5&-0.1&-0.1&2.9&0.0&0.6&0.6&-1.9&1.0&0.0\\

M6&-0.8&0.0&0.2&-0.5&-0.4&0.0&1.6&-1.1&2.1&0.0\\

M7&-0.5&0.5&-0.3&-0.5&0.1&0.7&0.0&0.4&1.0&-1.5\\

M8&-0.7&0.3&0.3&-0.1&-0.5&0.0&-0.5&0.0&0.8&0.0\\

M9&-0.4&0.1&0.2&0.0&0.4&0.1&0.9&0.9&0.0&4.5\\

M10&-1.0&0.0&-0.2&-0.2&-0.2&0.3&-1.3&-1.1&0.0&0.0\\

\hline
\end{tabular}
}
\caption{\footnotesize{Stability of FL performance when a single user leaves. M1-M10
  are manufacturers sorted in descending order by size. Each row
  represents a manufacturer that is leaving the federation. Each
  Column represents the difference between F1 values under full
  federation and this reduced federation for that manufacturer. }}
\vspace{-0.3in}
\label{leave-one-out-FL}
\end{table}

\begin{table}
\footnotesize{
\centering
\begin{tabular}{|c|c|c|c|c|c|c|c|c|c|c|}
\hline
&M1 & M2 & M3& M4& M5& M6& M7& M8& M9& M10\\\hline

M1&0.0&0.1&0.4&1.9&-2.4&1.4&2.9&-8.3&0.3&15.8\\

M2&-0.1&0.0&0.6&1.6&-1.5&-1.6&0.5&-2.5&1.4&22.5\\

M3&0.5&0.5&0.0&2.1&-1.7&0.2&-1.3&-1.2&-1.2&3.7\\

M4&-0.3&-0.3&0.2&0.0&-0.1&-4.3&0.7&-1.3&-0.4&18.7\\

M5&-0.1&0.0&-0.3&1.0&0.0&-0.3&-0.3&-1.9&-0.8&0.5\\

M6&-0.2&-0.5&0.3&1.6&-1.9&0.0&-1.5&-0.3&-0.5&4.2\\

M7&-0.5&0.1&0.3&2.2&-1.2&-2.8&0.0&-0.5&0.9&28.9\\

M8&0.5&-0.5&0.8&0.6&0.0&-4.0&-0.9&0.0&5.2&15.8\\

M9&-0.5&-0.5&0.3&1.0&-2.5&-3.3&-3.5&-2.4&0.0&4.1\\

M10&-0.1&-0.2&1.0&0.9&-1.8&-3.2&-0.1&-1.4&2.2&0.0\\

\hline
\end{tabular}
}
\caption{\footnotesize{Stability of Private FL with Fine Tuning performance when a
  single user leaves. M1-M10 are manufacturers sorted in descending
  order by size. Each row represents a manufacturer that is leaving
  the federation. Each Column represents the difference between F1
  values under full federation and this reduced federation for that
  manufacturer.}}
\vspace{-0.4in}
\label{leave-one-out-FT-DP-FL}
\end{table}

\subsection{Federation of Small Manufacturers}

Another scenario that we simulate is the one where only participants
with small amount of training data agree to collaborate. In this case,
we do not have the advantage of the large amount of training data from
any of the larger manufacturers.
%\autoref{small-fed-results-table} shows the results for this scenario.
To better understand if such a
federation is still advantageous, we compare the F1 values for small
manufacturers in two different scenarios: one, in which they are a
part of a large federation with all manufacturers, and second, in
which they are a part of a federation with only the small
manufacturers.  Figures~\ref{small-vs-large-fed-fl}
and~\ref{small-vs-large-fed-ft-dp-fl} show these comparisons for FL
and FT-DP-FL respectively. As is clear from the bar chart, even in the
case of a federation with just the small manufacturers, most of the
manufacturers benefit significantly from participating. In fact, the
performance of all manufacturers in the small federation closely
tracks their performance in the large federation, with one exception.
%In case of FT-DP-FL, small amounts of data lead to high $\sigma$
%values, which seems to make things less stable.

\begin{figure}[t]
\includegraphics[width=2.5in]{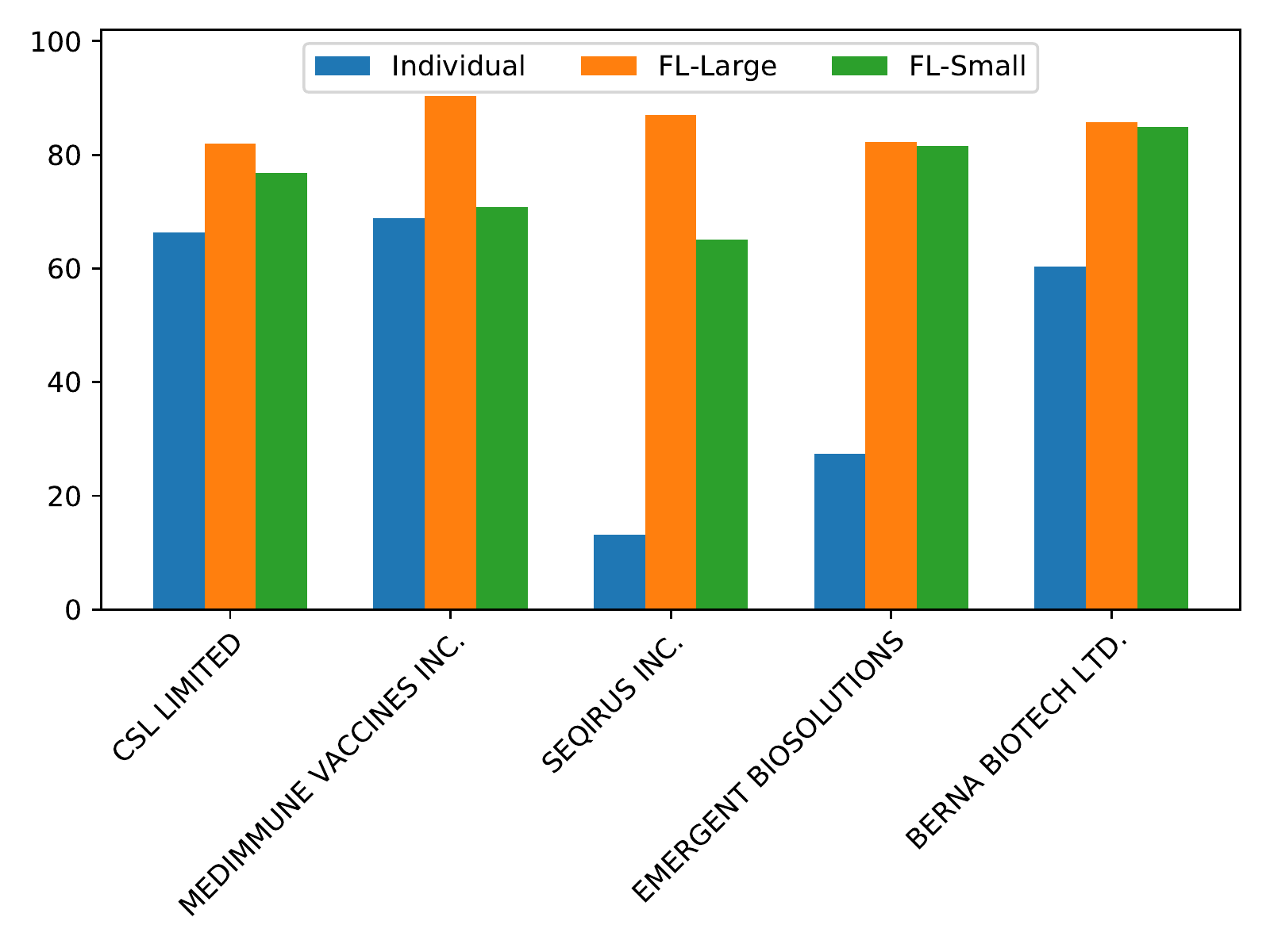}
\vspace{-0.28in}
\caption{\footnotesize{Comparison of FL F1 for small manufacturers when they are a part of a larger federation vs. a federation of only small manufacturers.}}
\label{small-vs-large-fed-fl}
\end{figure}

\begin{figure}[t]
\vspace{-0.2in}
\includegraphics[width=2.5in]{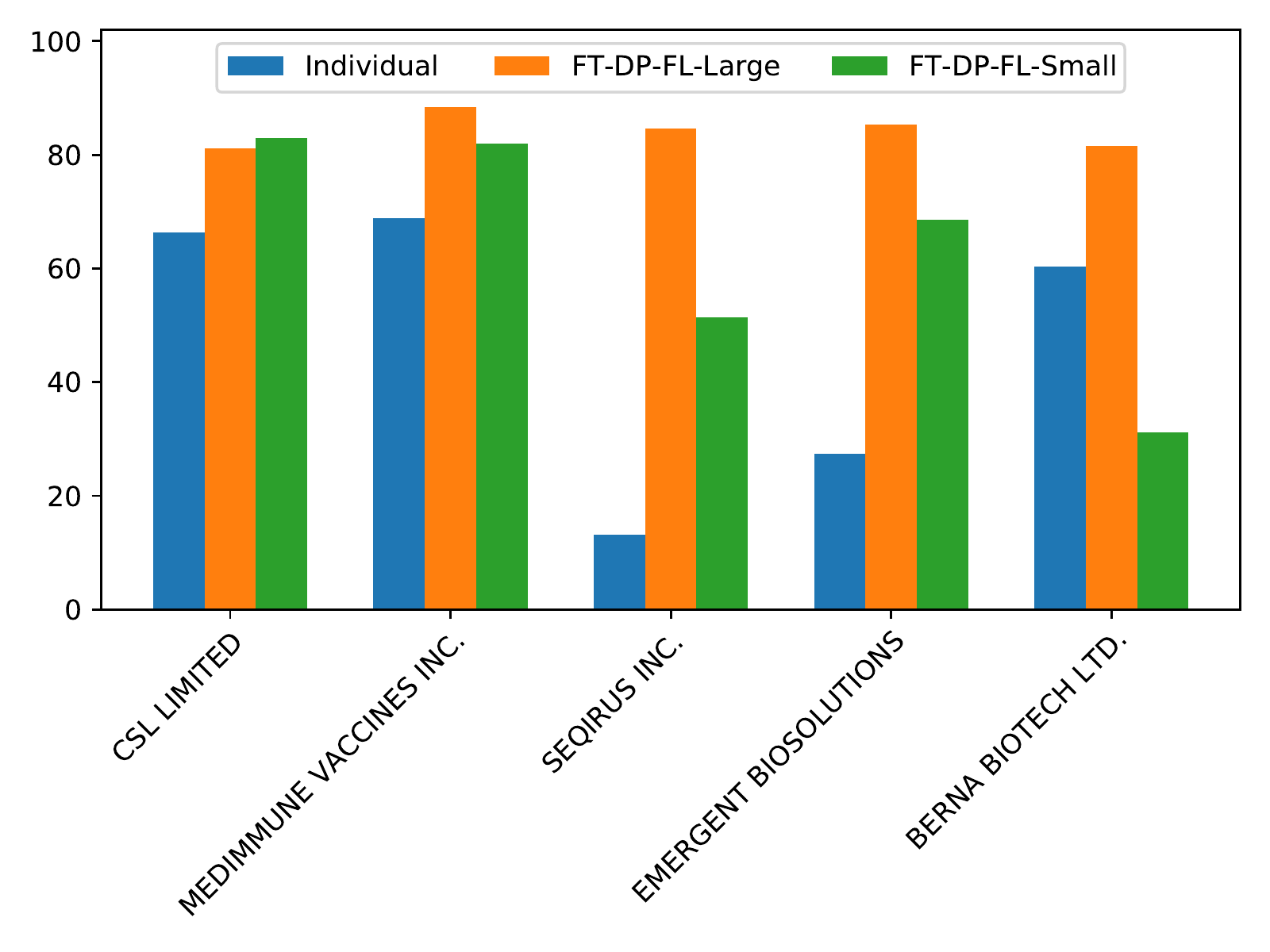}
\vspace{-0.25in}
\caption{\footnotesize{Comparison of FT-DP-FL F1 for small manufacturers when they are a part of a larger federation vs. a federation of only small manufacturers.}}
\vspace{-0.2in}
\label{small-vs-large-fed-ft-dp-fl}
\end{figure}

\subsection{Training Time}
 
All experiments were run on the Oracle Cloud Infrastructure cluster of
Tesla V100 GPUs running a job scheduling software. The GPUs were
either 1, 2 or 8 core, with 90G, 180G, 768G memory respectively. Here
we report the actual wall clock time for training different variants
of federated learning. Training the FL model took 7.95 minutes, while
training it and tuning it for each of the users in a serial fashion
took a total of 17.10 minutes. The DP-FL model took 505.84 minutes to
train by itself and 559.34 minutes with fine tuning.  The DP models
took over an order of magnitude of training time because during
training the DP noise injection code path computes and clips gradients
of individual data points in a training mini-batch before applying
gaussian noise to the averaged mini-batch gradients.  This is
necessary to ensure that the training algorithm respects the allotted
$\epsilon$ privacy budget over the training process.  Parallelization
of this component of our system using Goodfellow's
technique~\cite{goodfellow15} is the subject of future work.

\vspace{-0.1in}
\section{PAC Non-Identifiability}
\label{sec:pac-non-identifiability}

Inspired by Probably Approximately Correct (PAC) learning, London
recently introduced \emph{PAC Identifiability}~\cite{london20}, a new
privacy condition relevant to personalization in
FL~\cite{dinh20,fallah20,liang20,mansour20,peterson19,yu20}.
Informally, in a personalized FL setting, PAC identifiability
determines whether the private model used by a federation's user can
be leaked out to an adversarial federation server.  Learning a user's
private model can fundamentally compromise the user's privacy.  London
presents the case study of recommender systems, where the federation
server may be able to determine ratings choices made by a targeted
user.  It is critical for user privacy to determine if a given
personalization approach's user-local (private) model is PAC
identifiable.  To that end we now prove that FT-FL is not PAC
identifiable.

Let $G$ be the global model containing parameters $p_1,p_2,...,p_n$.
Let $L_u$ be the local (private) model for user $u$, and $D_u$ denote
the user's private data used to train $G$ and $L_u$.  We can
w.l.o.g. represent personalization in FL at user $u$ as follows:

\begin{equation}
  \Delta p_u = L_u(D_u) \oplus G(D_u)
\label{eq:pac1}
\end{equation}

where $\oplus$ is the personalization specific operator (algorithm)
that combines the local and global models' outputs to yield $\Delta
p_u$, the update to $G$'s parameters that is shipped back to the
federation server.

We use London's definition of PAC identifiability in his restricted
context of binary classification for a recommender system in our
proof.  However, our proof can be easily generalized to a richer
definition of PAC identifiability.
\begin{definition}
  A user $u$, using a given protocol (which may be stochastic), is
  \emph{PAC Identifiable} if, for any $\epsilon \in (0,1)$ and $\delta
  \in (0,1)$, with probability at least $1 - \delta$ over $T =
  poly(\epsilon^{-1},\delta^{-1})$ interactions with the server, the
  server can output an estimate $\hat{L}_u$, of the local model (after
  interaction), $L_u$, such that
  
  \hspace{0.2in}$1\over{I}$$\sum_{i \in \mathcal{I}}
  \mathbb{1}\{\hat{L}_u(D_u) \oplus G(D_u) \neq L_u(D_u) \oplus
  G(D_u)\} \le \epsilon$.
\label{def:pac-identifiability}
%\vspace{-0.1in}
\end{definition}
\noindent where, $\mathcal{I}$ is the set of $I$ items from the
catalog in the recommender system.  Informally, PAC identifiability
puts an upper bound $\epsilon$ on the number of disagreements between
the user's local model $L_u$, and its estimate $\hat{L}_u$ predicted
by an adversarial federation server.  In such cases, we say that
models $L_u$ and $\hat{L}_u$ are \emph{similar}.

While London~\cite{london20} describes a simple PAC identifiability
mechanism (protocol), more sophisticated mechanisms will be proposed
by researchers in the future.  Our proof of FT-FL's PAC
\emph{non-identifiability} is agnostic to such mechanisms.

Formally, let $A_G$ be a mechanism employed by the adversarial
federation server such that
\begin{equation}
  A_G(\Delta p_u) \triangleq \hat{L}_u
  \label{eq:pac2}
\end{equation}
where $\hat{L}_u$ is the estimate of $L_u$.  We say that $A_G$ is the
\emph{PAC identifiability mechanism} for $L_u$.

Clearly, in the process of deriving $\hat{L}_u$, $A_G(\Delta p_u)$
eliminates $G(D_u)$ or its effects from Equation~\ref{eq:pac1}.  Let
us call that operation $G^-$.  Therefore,
\begin{equation}
  G^-(\Delta p_u) = L_u(D_u) + \gamma
\label{eq:pac3}
\end{equation}
where $\gamma$ is the noise introduced by $G^-$ in the process of
eliminating the effects of $G(D_u)$ on $\Delta p_u$.  Thus,
\begin{equation}
  A_G(\Delta p_u) = U^-(G^-(\Delta p_u))
\label{eq:pac4}
\end{equation}
where $U^-$ maps $L_u(D_u) + \gamma$ to $\hat{L}_u$.  $\gamma$ must be
negligible enough to allow the PAC identifiability condition
(\ref{def:pac-identifiability}) to be satisfied.

\begin{lemma}
  In any setting where a model is trained by FL, and users fine-tune
  the model within their silo after training is complete, the
  adversarial federation server's PAC identifiability mechanism,
  $A_G$, yields a model that is similar to the \emph{\textsf{Null}}
  model (the model with all its parameters set to the value $0$):

  \hspace{1.0in}$A_G(\Delta p_u) = \mathcal{O}$
\end{lemma}

\begin{proof}
  As per Equation~\ref{eq:pac1}

  \hspace{1.0in}$\Delta p_u = L_u(D_u) \oplus G(D_u)$

  In case of FT-FL, $L_u(D_u)$ is completely missing from the above
  composition that yields the parameter update $\Delta p_u$.  In fact,
  $L_u(D_u)$ is computed \emph{after} the entire FL training process
  completes.  Recall however, that $L_u(D_u)$ is used by user $u$
  privately to make its post-training local predictions.  In effect,

  \hspace{1.0in}$\Delta p_u = \bar{0} \oplus G(D_u)$

  In fact, $\Delta p_u = G(D_u)$.  As a result, Equation~\ref{eq:pac3}
  evaluates to

  \hspace{1.0in}$G^-(\Delta p_u) = \bar{0} + \gamma$
  
  and as $U^-$ maps $\bar{0} + \gamma$ to $\hat{L}_u$, the latter is
  similar to the \textsf{Null} model $\mathcal{O}$.
  
\end{proof}

The following corollaries follow

\begin{corollary}
  FT-FL is not PAC identifiable.
\end{corollary}

\begin{corollary}
  FT-DP-FL is not PAC identifiable.
\end{corollary}

\noindent Since fine-tuning of FT-FL (and FT-DP-FL) follows
conventional training methodologies, the convergence proof of the fine
tuning component of FT-FL (and FT-DP-FL) is identical to standard
convergence proofs for stochastic gradient descent and similar
optimization algorithms.

\section{Conclusion}
\label{sec:conclusion}

Extracting mentions of vaccine adverse events using machine learning methods is an extremely urgent task right now. Federated Learning is a promising approach for breaking down organizational and geographical barriers to collaboration on building very effective models to solve this problem. Our work demonstrates that the loss of accuracy incurred through adding additional layers of privacy can be mitigated by introducing personalization. We show that manufacturers with dataset of all different sizes can benefit from participating in such a federation and that it is stable to potential real world changes. We also prove that adding personalization actually further enhances privacy. In the future, we would like to investigate other approaches to personalization applied to this problem domain.

\bibliographystyle{abbrv}
\bibliography{refs}

\end{document}